\documentclass[sigconf]{acmart}

\AtBeginDocument{%
	\providecommand\BibTeX{{%
			\normalfont B\kern-0.5em{\scshape i\kern-0.25em b}\kern-0.8em\TeX}}}

\copyrightyear{2022}
\acmYear{2022}
\setcopyright{acmcopyright}
\acmConference[KDD '22] {Proceedings of the 28th ACM SIGKDD Conference on Knowledge Discovery and Data Mining}{August 14--18, 2022}{Washington, DC, USA.}
\acmBooktitle{Proceedings of the 28th ACM SIGKDD Conference on Knowledge Discovery and Data Mining (KDD '22), August 14--18, 2022, Washington, DC, USA}
\acmPrice{15.00}
\acmISBN{978-1-4503-9385-0/22/08}
\acmDOI{10.1145/3534678.3539442}

\usepackage{hyperref}       % hyperlinks
\usepackage{url}            % simple URL typesetting
\usepackage{booktabs}       % professional-quality tables
\usepackage{amsfonts}       % blackboard math symbols
\usepackage{microtype}
\usepackage{graphicx}
\graphicspath{{figures/}}
\usepackage{subcaption}
\usepackage{ifthen}
\usepackage{colortbl}
\usepackage{mathtools}
\usepackage{multirow}
\usepackage{tabularx}
\usepackage{float}
\usepackage{algorithm}% http://ctan.org/pkg/algorithms
\usepackage{algorithmic}
\usepackage[inline]{enumitem}
\usepackage[english]{babel}
\usepackage{wrapfig}
\usepackage{color,soul}
\usepackage{changepage}
\usepackage{xcolor}

\DeclareMathOperator*{\argmin}{arg\,min}

\usepackage{amsthm}

\DeclareMathSizes{10}{9}{7}{5}
\newtheorem{definition}{Definition}

\newtheorem{theorem}{Theorem}
\newtheorem{lemma}{Lemma}
\newtheorem{assumption}{Assumption}

\setcopyright{acmcopyright}
\begin{document}
\title{Saliency-Regularized Deep Multi-Task Learning}

\author{Guangji Bai}
\affiliation{%
	\institution{Emory University}
	\streetaddress{Department of Computer Science}
	\city{Atlanta}
	\state{GA}
}
\email{guangji.bai@emory.edu}

\author{Liang Zhao}
\authornote{Corresponding Author}
\affiliation{%
	\institution{Emory University}
	\streetaddress{Department of Computer Science}
	\city{Atlanta}
	\state{GA}
}
\email{liang.zhao@emory.edu}

\begin{abstract}
Multi-task learning (MTL) is a framework that enforces multiple learning tasks to share their knowledge to improve their generalization abilities. While shallow multi-task learning can learn task relations, it can only handle pre-defined features. Modern deep multi-task learning can jointly learn latent features and task sharing, but they are obscure in task relation. Also, they pre-define which layers and neurons should share across tasks and cannot learn adaptively. To address these challenges, this paper proposes a new multi-task learning framework that jointly learns latent features and explicit task relations by complementing the strength of existing shallow and deep multitask learning scenarios. Specifically, we propose to model the task relation as the similarity between tasks’ input gradients, with a theoretical analysis of their equivalency. In addition, we innovatively propose a multi-task learning objective that explicitly learns task relations by a new regularizer. Theoretical analysis shows that the generalizability error has been reduced thanks to the proposed regularizer. Extensive experiments on several multi-task learning and image classification benchmarks demonstrate the proposed method’s effectiveness, efficiency as well as reasonableness in the learned task relation patterns. 
\end{abstract}

\begin{CCSXML}
<ccs2012>
<concept>
<concept_id>10010147.10010257.10010321.10010337</concept_id>
<concept_desc>Computing methodologies~Regularization</concept_desc>
<concept_significance>500</concept_significance>
</concept>
<concept>
<concept_id>10010147.10010257.10010258.10010262</concept_id>
<concept_desc>Computing methodologies~Multi-task learning</concept_desc>
<concept_significance>500</concept_significance>
</concept>
</ccs2012>
\end{CCSXML}

\ccsdesc[500]{Computing methodologies~Multi-task learning}
\ccsdesc[500]{Computing methodologies~Regularization}

\keywords{Multi-task Learning, Task Relation, Saliency Detection}

\maketitle
    
\section{Introduction}
\emph{Multi-task learning} (MTL, \cite{caruana1997multitask}) is an important research domain based on the idea that the performance of one task can be improved using related tasks as inductive bias. While traditional shallow MTL methods can fit the models for individual tasks and learn task relations, they do not focus on generating features from scratch and instead rely on pre-defined and explicit features~\cite{zhang2021survey,torres2021sign}. More recently, deep representation learning empowers MTL to go "deep" by equipping it with the capacity to generate features while fitting the tasks' predictive models. Deep MTL is usually categorized according to the ways of correlating tasks' models into two major types: \emph{hard-parameter sharing} and \emph{soft-parameter sharing}. Hard-parameter sharing methods~\cite{zhang2014facial,long2017learning} essentially hard-code which part of neurons or layers to share and which part does not for different tasks instead of doing it adaptively. Moreover, they usually share the layers for representation learning (e.g., convolutional layers) but not those for decision making (e.g., fully-connected layers for classification). On the other hand, soft-parameter sharing methods~\cite{duong2015low,misra2016cross} do not require to hard-code the sharing pattern but instead build individual models for each task and "softly" regularize the relatedness among them. Hence, soft-parameter sharing has better flexibility in learning the task relation, while may not be efficient since its model parameters increase linearly with the number of tasks. Hard-parameter sharing, by contrast, is more "concise" but requires pre-define which parts are shared or not.

Therefore, although MTL is a long-lasting research domain, it remains a highly challenging and open domain that requires significantly more efforts to address challenges such as the trade-off between model flexibility and conciseness of hard- and soft-parameter sharing mentioned above. Although more recently, there have come a few attempts trying to alleviate the dilemma, such as those regularizing task relationships in task-specific layers in hard-parameter sharing to achieve knowledge transfer in unshared layers~\cite{long2017learning} and those adaptively learning which part to share or not by methods like branching~\cite{lu2017fully} or Neural Architecture Search~\cite{sun2019adashare}, the research frontiers still suffer from several critical bottlenecks, including \textbf{(1) Difficulty in regularizing deep non-linear functions of different tasks.} Adaptively learning task relation requires regularizing different tasks' predictive functions, which, however, are much harder to achieve for nonlinear-nonparametric functions since it requires regularizing in the whole continuous domain of input. To work around it, existing works~\cite{long2017learning,strezoski2019learning} typically resort to a \emph{reduced problem} which is to regularize the neural network parameters. Notice that this reduction deviates from the original problem and is over-restricted. For example, first, two neural networks with different permutations of latent neurons can represent the same function. Moreover, even if they have different architectures, they can still possibly represent the same function~\cite{lecun2015deep}. This gap deteriorates the model's generalizability and effectiveness. \textbf{(2) Lack of interpretability in joint feature generation and task relation learning.} Despite incapability of generating features, shallow MTL enjoys good interpretability since they learn explicit task correlations via how the hand-crafted features are utilized. However, in deep MTL, the generated features do not have explicit meaning and how the black-box models relate to each other is highly obscure. It is imperative yet challenging to increase the interpretability of both generated features and task relation. \textbf{(3) Difficulty in theoretical analysis.} While there are fruitful theoretical analyses on shallow MTL, such as on generalization error~\cite{baxter2000model} and conditions for regularized MTL algorithms to satisfy representer theorems~\cite{argyriou2007spectral}, similar analyses meet strong hurdles to be extended to deep MTL due to the difficulty in reasoning about neural networks whose feature space is given by layer-wise embeddings~\cite{wu2020understanding}. It is crucial to enhance the theoretical analyses on the model capacity and theoretical relation among different deep MTL models.

This paper proposes a new \textbf{\underline{S}}aliency-\textbf{\underline{R}}egularized \textbf{\underline{D}}eep \textbf{\underline{M}}ulti-task \textbf{\underline{L}}earning (\textbf{SRDML}) framework to solve the challenges mentioned above. First, we reconsider the feature weights in traditional linear multitask learning as the input gradient and then generalize the feature learning into the non-linear situation by borrowing the notion of saliency detection. Second, we recast the task relation problem as the similarity among saliency regions across tasks so as to regularize and infer the task relation. Third, to validate our hypothesis, we have given a theoretical analysis of their equivalency. Meanwhile, we also provide theoretical analysis on how the proposed regularization helps reduce the generalization error. Finally, we demonstrate our model's effectiveness and efficiency on synthetic and multiple large-scale real-world datasets under comparison with various baselines.

%To this end, we reconsider the feature weights in traditional linear multitask learning as the input gradient and then generalize the feature learning into non-linear situation by borrowing the notion of saliency detection. The saliency provide transparency and interpretability on the latent features and how they determine the prediction. We then recast the task relation problem as the similarity among saliency regions across tasks so as to regularize and infer the task relation. To validate our hypothesis, we have given theoretical analysis on the equivalency. Meanwhile, we also provide theoretical analysis on how the proposed regularization helps reduce the generalizability error.

%Inspired by recent work in saliency-based explanation methods (\cite{zhou2016learning, selvaraju2017grad}), the saliency map with respect to each object class will highlight the region of the correspond object while pay little or no attention to other regions that does not help the task, and this rule could potentially play an important role in multi-task learning in computer vision domain. For example, given an image of human face and two tasks which detects the age and the gender respectively, both tasks will possibly focus on the region of hair because the older usually have less hair than the youngsters and female typically have longer hair than male. In other words, the object hair can be a key input feature for both tasks. (\mytodo{where to put this paragraph})

\section{Related Work}

\noindent\textbf{Multi-task learning (MTL).} Readers may refer to~\cite{zhang2021survey, crawshaw2020multi} for a more comprehensive survey on MTL. Before the popularity of deep learning, traditional MTL usually focuses on hand-crafted features and can be generally divided into two categories: \textbf{1).} \emph{multi-task feature learning}, which aims to learn a shared/similar feature selection, latent space, or model parameters~\cite{argyriou2008convex, evgeniou2007multi}. \textbf{2).} \emph{multi-task relation learning}, which aims to quantify task relatedness via task clustering~\cite{jacob2008clustered, kumar2012learning} or task co-variance~\cite{evgeniou2004regularized, zhang2012convex}. However, they rely on hand-crafted features and the separation of feature generation and task learning may result in sub-optimal performance.

More recently, MTL takes advantage of the advancement of deep neural networks which can directly take raw, complex data such as images, audio, texts, and spatial-temporal data~\cite{zhao2017feature,zhao2018distant,gao2018incomplete,gao2019incomplete} and learn in an end-to-end manner. Deep MTL integrates feature generation and task learning and simultaneously learns both of them~\cite{bengio2013representation}. In this domain, \emph{hard-parameter sharing}~\cite{ouyang2014multi,zhang2014facial} requires to hard-code which part of the network is shared and which is not. Existing work usually shares the lower-level layers for representation learning (e.g., convolutions) while make higher-level layers (e.g., those for classification) separated across tasks. Some existing works extend hard-parameter sharing by considering Neural Architecture Search~\cite{elsken2019neural} like~\cite{lu2017fully}, \cite{sun2019adashare}, and \cite{guo2020learning}. \emph{Soft-parameter sharing} based method has better flexibility where each task has its own models and regularization is used to enforce task relatedness by aligning their model parameters~\cite{duong2015low,misra2016cross}. To achieve both hard-parameter sharing's conciseness and efficiency and soft-parameter sharing's flexibility, some recent work~\cite{long2017learning,strezoski2019learning} shares the representation learning layers while exploits task relations in task-specific layers.

\noindent\textbf{Saliency detection.} Saliency detection is to identify the most important and informative part of input features. It has been applied to various domains including CV~\cite{goferman2011context,gao2022res}, NLP~\cite{li2015visualizing,ren2019generating}, etc. The salience map approach is exemplified by~\cite{zeiler2014visualizing} to test a network with portions of the input occluded to create a map showing which parts of the data actually have an influence on the network output. In~\cite{simonyan2013deep}, a salience map can be created by directly computing the input gradient. Since such derivatives can miss important aspects of the information that flows through a network, a number of other approaches have been designed to propagate quantities other than gradients through the network. In CV domain, Class Activation Mapping (CAM,~\cite{zhou2016learning}) modifies image classification CNN architectures by replacing fully-connected layers with convolutional layers and global average pooling,  thus achieving class-specific feature maps. Grad-CAM~\cite{selvaraju2017grad} generalizes CAM by visualizing the linear combination of the last convolutional layer's feature map activations and label-specific weights, which are calculated by the gradient of prediction score w.r.t the feature map activations. Grad-CAM invokes different versions of backpropagation and/or activation, which results in aesthetically pleasing, heuristic explanations of image saliency. While there exist some other saliency-based methods along this research line including Guided Propagation~\cite{springenberg2014striving}, Deconvolutional Network~\cite{zeiler2014visualizing}, etc, they are designed only for specific architectures like ReLU Network for Guided Propagation.

\begin{figure}[t!]
  \begin{center}
    \includegraphics[width=0.47\textwidth]{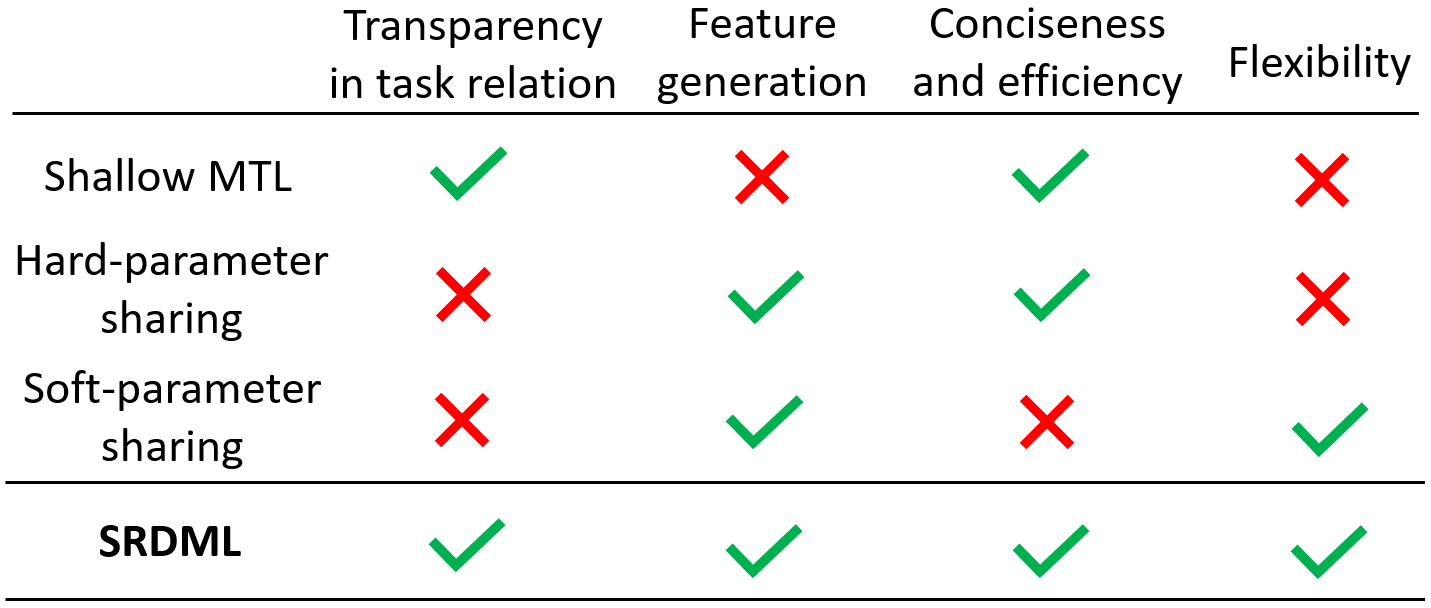}
  \end{center}
  \vspace{-0.2cm}
  \caption{Comparison over different MTL methods.}
  \vspace{-0.2cm}
  \label{fig:pro and con}
  \vspace{-0.2cm}
\end{figure}

\begin{figure*}[t!]
  \includegraphics[width=0.98\textwidth]{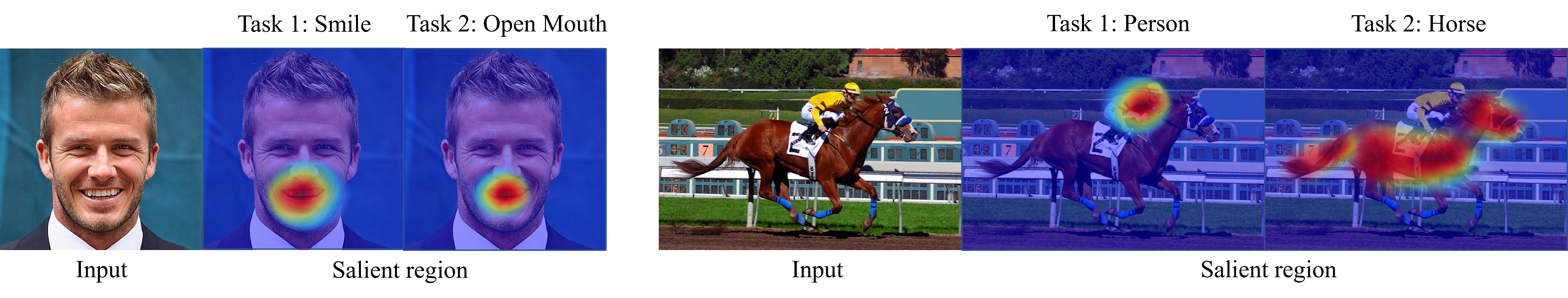}
  \vspace{-2mm}
  \caption{Illustrative examples of relation between saliency and task similarity. Left: Two tasks are to detect whether the man is smiling and his mouth is open. The salient regions for two tasks are both around the mouth. Right: Two tasks are to detect the horse and person. The salient regions are close to each other, indicating the potential similarity between the tasks.}
  \label{fig:saliency and task similarity}
\vspace{-0.2cm}
\end{figure*}

\section{Proposed Method}
In this section, we introduce our proposed Saliency-regularized Deep Multi-task Learning (SRDML) method. We first review the pros and cons for each MTL method and describe our main motivation, then formally introduce our model and its objective function.

\subsection{Problem Formulation}
Consider a multi-task learning problem with $T$ tasks such that a dataset $\{\mathbf{X}, \mathbf{Y}_1, \mathbf{Y}_2, \cdots, \mathbf{Y}_T\}$ is given with i.i.d training samples $\mathbf{X}_t = \{\mathbf{x}^{(t)}_1,\mathbf{x}^{(t)}_2$, $\cdots,\mathbf{x}^{(t)}_n\}$, $\mathbf{Y}_t = \{\mathbf{y}^{(t)}_1,\mathbf{y}^{(t)}_2,\cdots,\mathbf{y}^{(t)}_n\}$, where $n$ is the sample size and $(\mathbf{x}^{(t)}_i,\mathbf{y}^{(t)}_i)$ is a pair of input and label such that $\mathbf{x}^{(t)}_i\in\mathcal{X}$ and $\mathbf{y}^{(t)}_i\in\mathbb{R}$, $\forall\; i = 1, 2, \cdots, n$ and $t=1,2,\cdots,T$. 

%$(\mathbf{x}_i,\mathbf{y}^{(t)}_i)$ is a pair of input and label such that $\mathbf{x}^{(t)}_i\in\mathcal{X}$ and $\mathbf{y}^{(t)}_i\in\mathbb{R}$, $\forall\; i = 1, 2, \cdots, n$. 

Given a predictor $g$ which factorizes as $g = f \circ h$, where "$\circ$" stands for functional composition. The function $h: \mathcal{X}\rightarrow \mathbb{R}^{K}$ is called the feature or representation extraction part and is shared for all tasks, while $f: \mathbb{R}^{K}\rightarrow \mathbb{R}$ is a function defined on $\mathbb{R}^{K}$, a predictor specialized to each task at hand. $K$ denotes the latent representation or feature-map dimensions. We further assume that each task shares the same input feature $\mathbf{x}$, i.e., $\mathbf{x}^{(1)}=\mathbf{x}^{(1)}=\cdots=\mathbf{x}^{(T)}$, which is very commonly seen in deep MTL problems such as multi-task image classification task in the Computer Vision domain.

Our \emph{goal} is to build a deep architecture for learning multiple tasks $\mathbf{y}^{(t)}_i = g_{t}(\mathbf{x}_i),\;t = 1, 2, \dots ,T$ which jointly generates semantic features and learns task relation to correlate different tasks with interpretability. This goal poses significant challenges to existing work: \textbf{1).} Directly regularizing the prediction function of different tasks is extremely hard. Existing work considered a reduced problem by regularizing the feature weights of different $f_t$ which is over-restricted. \textbf{2).} How to learn interpretable task relations with deep/implicit features is still unclear. \textbf{3).} Theoretical analysis is rare in deep MTL due to the non-linear and non-parametric functions of $h$ and $f$. To jointly solve these challenges, we reconsider the feature weights in shallow MTL as \emph{input gradient}, i.e., $\partial f(x)/\partial x,\; x\in \mathbb{R}^K$, and generalize the feature learning into the deep network by considering the saliency detection methods.

\subsection{Motivations}
We propose a simple framework that can innovatively achieve all the goals, as shown in Figure~\ref{fig:pro and con}.

To achieve model conciseness and efficiency as well as task relatedness flexibility, we share the representation learning layers and learn task relationships in task-specific layers. This is based on essential neuro-inspirations: human sensory organs and retina are the same for all different tasks (meaning the convolution layers are shared). On the other hand, the working memory will leverage the long-term memory for each task, and related tasks will have related memory (i.e., model), and their relatedness can be considered as the similarities of activation patterns for different tasks, namely the similarity among the saliency maps for different tasks.

Then, the next question is how to regularize the relation among different tasks, namely, how to regularize the (dis)similarity of the predictive functions of different tasks. As mentioned above, it is problematic to directly regularize the neuron network parameters due to their gap with the actual function. For example, neural networks with different architectures or neuron permutations could represent the same function. Therefore, this motivates us to explore an innovative alternative so that we can more easily work towards the space of \emph{functional}. Specifically, we propose to regularize first-order derivatives with respect to the input of different tasks. This new strategy has two crucial merits: First, it is mathematically equivalent to directly regularizing the function without the gap in existing works mentioned above. Second, it also finds inspiration from the saliency map domain and comes with strong interpretability in how tasks correlate.

\noindent\textbf{Key Merit 1: Regularizing task functions without theoretical gap.} Specifically, Theorem~\ref{thm:model assumption correctness} below tells us that enforcing multiple tasks to have similar input gradients is equivalent to encouraging those tasks themselves to be similar. 
\begin{theorem}
\label{thm:model assumption correctness}
Define $\mathcal{F}\vcentcolon=\{f\in \textbf{C}^1: f(0)=0\}$, where $\textbf{C}^k$ is the family of functions with $k^{th}$-order continuous derivatives for any non-negative integer $k$. Given $f_1,f_2\in \mathcal{F}$, we have:
%\vspace{-0.1cm}
\begin{equation}
    f_1 = f_2 \quad {\textbf{if and only if}}\quad f_1^{\prime}(x) = f_2^{\prime}(x),\; \forall x\in\mathcal{X}
\end{equation}
\end{theorem}
%\vspace{-0.3cm}
\begin{proof}
Please refer to the appendix for the formal proof.
\end{proof}

Our analysis above allows us to regularize the prediction functions of different tasks in the \emph{functional} space instead of parameter space. The assumption over function family $\mathcal{F}\vcentcolon=\{f\in \textbf{C}^1: f(0)=0\}$ is reasonable in practice since an all-zero input $x$ simply corresponds to a "black" picture, and for any tasks we assume a black picture contains no useful information and should be classified as the negative sample (i.e., ground-truth label should be 0).

\noindent\textbf{Key Merit 2: Inspiration from saliency map and enhancement of interpretability.} Evaluating task relation with derivative similarity has justification from a saliency perspective. Saliency is a derivative of the prediction score w.r.t. input features, and it denotes the semantic features that influence the prediction most. In addition, similar tasks tend to have similar saliency, while dissimilar tasks tend to have dissimilar saliency. As shown in Figure~\ref{fig:saliency and task similarity}, we enforce higher-level semantic features as saliency.

Many previous work have asserted that deeper representations in a CNN capture higher-level visual constructs~\cite{bengio2013representation}. Furthermore, convolutional layers naturally retain spatial information which is lost in fully connected layers, so we expect the last convolutional layers to have the best compromise between high-level semantics and detailed spatial information. By following a recent work called Grad-CAM~\cite{selvaraju2017grad}, we use the gradient information flowing into the last convolutional layer of the CNN to capture the saliency map to each neuron for a particular task or class of interests.

\subsection{Objective Function}
We first give a formal definition of saliency. For example, in computer vision, given an input image $I$, a classification ConvNet $f$ predicts $I$ belongs to class $c$ and produces the class score $f_{c}(I)$ (\emph{abbrev.} $f_c$). Let $A$ be the feature map activations of the last convolutional layer. We are curious about the rank of each pixel in $A$ based on their importance, which is referred to as saliency. The relationship between $f_c$ and $A$ is highly non-linear due to the non-linearity in $f$. In this case, we use the first-order derivatives i.e., $\partial f_c / \partial A$ to approximate the saliency map, which basically reflects the contributions of different pixels in $A$ to the prediction $f_c$. 

The objective function of SRDML is defined as follow:
\begin{equation}
\begin{split}
  &\quad\quad\min_{h,f_1,\cdots,f_T,\xi} \sum\nolimits_{t=1}^T\mathcal{L}_{t}(f_t(h(\mathbf{X})),\mathbf{Y}_{t}),\; \text{s.t.\;} \\ 
  &\forall\; i,j,\; {dist(\nabla_{A} f_i, \nabla_{A} f_j) \leq \xi_{ij}\;}, \;\sum\nolimits_{1\leq i < j \leq T}\;\xi_{ij}\leq \alpha
\end{split}
\label{eq:objective funciton with constraint with relaxation}
\end{equation}
where $i$, $j$ are task indexes with $1\leq i < j \leq T$, $A = h(\mathbf{X})$ is the feature map activations from the last convolutional layer of $h$, and $\nabla_{A}f_t$ is the first-order derivative of function $f_t$ with respect to $A$, i.e., $\partial f_t / \partial A$. The $dist(\cdot)$ function here can be any distance measure including commonly-used ones like $\ell_1$, $\ell_2$, etc, and any potential normalization on the input gradient can also be embeded in $dist(\cdot)$.

\begin{figure}[t!]
  \begin{center}
    \includegraphics[width=0.47\textwidth]{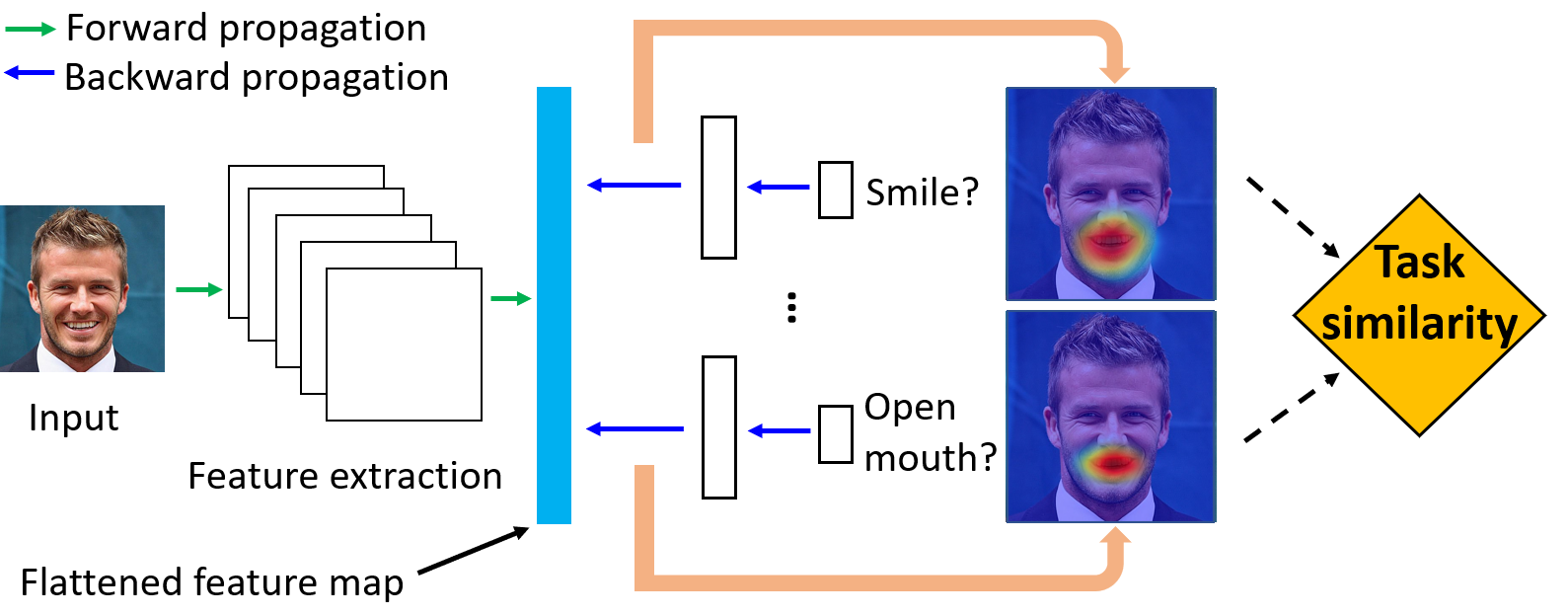}
  \end{center}
  \vspace{-0.2cm}
  \caption{A high level overview of SRDML architecture.}
  \label{fig:SRDML architecture}
  \vspace{-0.3cm}
\end{figure}

To adaptively learn the task relations, we introduce $\{\xi_{ij}\}_{1\leq i < j \leq T}$, which is a set of learnable \emph{slack} variables for each pair of tasks and $\alpha$ is a hyperparameter for controlling the overall level of slacking. Notice each $\xi_{ij}$ can only take non-negative value and this is guaranteed by the inequality constraint and the non-negative norm.

Directly optimizing Eq.~\ref{eq:objective funciton with constraint with relaxation} could be difficult due to the constraint. By utilizing Lagrangian method, we further transform Eq.~\ref{eq:objective funciton with constraint with relaxation} into a regularized form as follow: 
\vspace{-0.1cm}
\begin{equation}
\begin{split}
  &\min_{h,f_1,\cdots,f_T,\omega} \sum\nolimits_{t=1}^T \mathcal{L}_t(f_t(h(\mathbf{X})),\mathbf{Y}_{t}) \\ & \quad\quad\quad\quad\; + \lambda\cdot\sum\nolimits_{1\leq i < j \leq T}\; \omega_{ij}\cdot dist(\nabla_{A} f_i, \nabla_{A} f_j) \\
  &\text{s.t.,  } {\forall\; i,j,\;\omega_{ij}\geq 0\; \text{and}\;\; \sum\nolimits_{1\leq i < j \leq T}\;\omega_{ij}\geq \beta}  
\end{split}
\label{eq:objective funciton with regularizer}
\end{equation}
where $\{\omega_{ij}\}_{1\leq i < j \leq T}$ is a set of learnable parameters to explicitly model task relationship during the multi-task training, and $\lambda$ is the regularization coefficient. Our Eq.~\ref{eq:objective funciton with regularizer} is motivated by the \emph{graph regularization}~\cite{evgeniou2004regularized,evgeniou2005learning}, where each node corresponds to a specific task and $\omega_{ij}$ represents the weight for the edge between task $i$ and task $j$, so a graph-structure task relationship can be adaptively learned by SRDML. We rearrange the non-negative constraints over $\omega$ and apply normalization onto $\{\omega_{ij}\}_{1\leq i < j \leq T}$ to further simplify the constraints as follow:
\begin{equation}
\begin{split}
  &\min_{h,f_1,\cdots,f_T,\omega \succ 0} \sum\nolimits_{t=1}^{T} \mathcal{L}_t(f_t(h(\mathbf{X})),\mathbf{Y}) \\ &\quad\quad\quad\quad\quad\; + \lambda\cdot\sum\nolimits_{1\leq i < j \leq T} \frac{\omega_{ij}}{W}\cdot dist(\nabla_{A} f_i, \nabla_{A} f_j)
\end{split}
\label{eq:objective funciton with normalization}
\end{equation}
\vspace{0.2cm}
where $W = {\textstyle\sum}_{1\leq i < j \leq T}\;\omega_{ij}$. Thanks to our normalization trick, the overall objective of SRDML is differentiable and can be trained in an end-to-end manner. We use standard gradient descent (e.g., Adam~\cite{kingma2014adam}) to solve Eq.~\ref{eq:objective funciton with normalization}, where we aim to learn multiple tasks and the task relationship simultaneously. Although the normalization trick introduced in Eq.~\ref{eq:objective funciton with normalization} no longer guarantees that the hard constraint of the lower bound of all $\omega_{ij}$ can be strictly satisfied, our empirical results show that our normalization trick works well in practice and SRDML can capture reasonable task relationship by optimizing Eq.~\ref{eq:objective funciton with normalization} with finetuned hyperparameters. 

A general overview of SRDML architecture can be found in Figure~\ref{fig:SRDML architecture}. First, the input image is fed into a \emph{shared} feature extractor, which is implemented by a sequence of convolutional layers. Right after the feature extraction process, we obtain a set of flattened feature maps (shown as the blue bar in Figure~\ref{fig:SRDML architecture}), which contains high-level semantic information with respect to the original image~\cite{selvaraju2017grad}. On top of the feature map, each task-specific head will first calculate the saliency map with respect to its own prediction. Based on the saliency map for all the tasks, the task similarity can be calculated via some distance measure. Note that our overall framework is differentiable and can be trained in an end-to-end manner.   

Last, how to share the convolutional layers is \emph{orthogonal} to the focus of our paper because our SRDML focuses on task-specific layers instead of representation learning layers. This also implies whichever the best choice for convolutional layer sharing strategy can be utilized to work for our model. Our empirical results also demonstrated the reasonableness of the sharing policy that we used in this paper.

% \subsection{Discussion}
% Most existing works in deep MTL consider the sharing of all hidden representation learning layers like [1,3,4]. The recent research topic which aims at learning how to share the convolutional layers (e.g., [5,6,7]) is still an open problem and has some potential drawbacks including 1) Extra computational cost in learning the parameters for the sharing policy. 2) Extra parameters will increase the overall model complexity and may require more data samples to train. 3) Due to the flexibility and complexity in architecture learning, the theoretical guarantee for this line of works is limited and harder to pursue. How to share the convolutional layers is orthogonal to the focus of our paper because our SRDML focuses on task-specific layers instead of representation learning layers. This also means whichever the best choice for convolutional layer sharing strategy can be utilized to work for our model. Empirical results also demonstrated the reasonableness of the sharing policy that we used in this paper. Specifically, SRDML achieved superior performance over baselines (Tables 1 and 2) as well as reasonable task relations (Figures 4 and 7) on various widely used benchmarks.

\section{Theoretical Analysis}

In this section, we present the theoretical analyses of our SRDML model. First, we prove that our proposed regularizer can help reduce the generalization error. Second, we formally analyze the relation between SRDML and other MTL methods. We put all formal proofs in the appendix due to the limited space.

\subsection{Generalization Error Bound}
Here we show the generalization bound of our model. Our main contribution here is we proved that \textbf{our proposed regularization term can help reduce the generalization error.} 

% Recall Eq.~\ref{eq:objective funciton with constraint with relaxation}, where SRDML solves the following constrained problem: \textcolor{red}{consider remove Eq.5}
% %\vspace{-0.1cm}
% \begin{equation}
% \begin{split}
%   &\;\min_{h,f_1,\cdots,f_T,\xi} \frac{1}{nT} \sum\nolimits_{t=1}^T\sum\nolimits_{i=1}^n \mathcal{L}_{t}(f_t(h(\mathbf{x}_i)),\mathbf{y}^{(t)}_i), \;\text{s.t.\;} \\ 
%   &\forall\; i,j,\; {dist(\nabla_{A} f_i, \nabla_{A} f_j) \leq \xi_{ij}\;}, \;\sum\nolimits_{1\leq i < j \leq T}\;\xi_{ij}\leq \alpha
% \end{split}
% \label{eq:objective funciton GEB 1}
% %\vspace{-0.1cm}
% \end{equation}
% %\vspace{-0.1cm}
For simpler notation, define
\begin{equation}
\begin{split}
    \mathcal{F}_{\epsilon(\alpha)} \vcentcolon= &\big\{\textbf{f} \in \mathcal{F}^{T}:\; \forall\; 1\leq i < j \leq T,\; x\in \mathcal{X},\\
    &\quad dist(\nabla_{x} f_i, \nabla_{x} f_j) \leq \epsilon_{ij},\; \sum\nolimits_{1\leq i < j \leq T}\;\epsilon_{ij} \leq \alpha\big\}
\end{split}
\label{eq:F def}
\end{equation}
where $\textbf{f}=(f_1,f_2,\cdots,f_T)$ is the vectorization of each task's function, and $\{\epsilon_{ij}\}_{1\leq i < j \leq T}$ is a set of global slack variables. Hence, the optimization problem of Eq.~\ref{eq:objective funciton with constraint with relaxation} can be simplified as
\begin{equation}
%\vspace{-0.1cm}
\begin{split}
  &\min_{h\in\mathcal{H},\textbf{f}\in\mathcal{F}_{\epsilon(\alpha)}} \frac{1}{nT} \sum\nolimits_{t=1}^T\sum\nolimits_{i=1}^n \mathcal{L}_{t}(f_t(h(\mathbf{x}_i)),\mathbf{y}^{(t)}_i) 
\end{split}
\label{eq:objective funciton 3}
\end{equation}

Before introducing the theorem, we make the following standard assumptions over the loss function:
\begin{assumption}[\cite{maurer2016benefit}]
\label{ass:obj_function}
The loss function $\mathcal{L}$ has values in $[0,1]$ and has Lipschitz constant 1 in the first argument, i.e.: 
\begin{enumerate*}
    \item $\mathcal{L}(y,y^{\prime}) \in [0,1]$ 
    \item $\mathcal{L}(y,y^{\prime}) \leq y, \;\forall\;{y^{\prime}}$.
\end{enumerate*}
\end{assumption}
\noindent Different Lipschitz constants can be absorbed in the scaling of the predictors and different ranges than $[0, 1]$ can be handled by a simple scaling of our results.

\begin{definition}[Expected risk, Empirical risk]
Given any set of function $h,f_1,\cdots,f_T$, we denote the expected risk as:
\vspace{-0.1cm}
\begin{equation}
\begin{split}
\mathcal{E}(h,f_1,\cdots,f_T) \coloneqq \frac{1}{T} \sum\nolimits_{t=1}^T\mathbb{E}_{(X,Y)\sim{\mu_t}}[\mathcal{L}_{t}(f_t(h(X)),Y)]
\end{split}
\label{eq:expected risk}
\end{equation}
Given the data $\textbf{Z}=(\textbf{X},\textbf{Y})$, the empirical risk is defined as:  
\vspace{-0.1cm}
\begin{equation}
\begin{split}
\mathcal{\hat{E}}(h,f_1,\cdots,f_T|\textbf{Z}) \coloneqq \frac{1}{T} \sum\nolimits_{t=1}^T\frac{1}{n}\sum\nolimits_{i=1}^n\mathcal{L}_{t}(f_t(h(\mathbf{x}_i)),\mathbf{y}^{(t)}_i)
\end{split}
\label{eq:empirical risk}
\end{equation}
\end{definition}

\begin{definition}[Global optimal solution, Optimized solution]
Denote $(h^*,\textbf{f}^*)$ as the global optimal solution of the expected risk: 
\begin{equation}
    (h^*,\textbf{f}^*) \coloneqq \argmin_{h\in\mathcal{H}, \textbf{f}\in \mathcal{F}_{\epsilon(\alpha)}}\mathcal{E}(h,f_1,\cdots,f_T)
\end{equation}
and $(\hat{h}, \hat{\textbf{f}})$ as the optimized solution by minimizing the empirical risk:
\begin{equation}
    (\hat{h}, \hat{\textbf{f}}) \coloneqq \argmin_{h\in\mathcal{H}, \textbf{f}\in \mathcal{F}_{\epsilon(\alpha)}} \mathcal{\hat{E}}(h,f_1,\cdots,f_T|\textbf{Z})
\end{equation}
\end{definition}

The following theorem provides theoretical guarantee of our proposed method's generalizability.

\begin{theorem}[Generalization Error]
\label{thm:error bound}
Let $\delta>0$ and $\mu_1,\mu_2,\dots,\mu_T$ be the probability measure on $\mathcal{X} \times \mathbb{R}$. With probability of at least $1-\delta$ in the draw of $\mathbf{Z}=(\mathbf{X},\mathbf{Y})\sim\prod_{t=1}^{T}{\mu_{t}^{n}}$, we have:
\begin{equation}
\begin{aligned}
&\mathcal{E}(\hat{h}, \hat{\textbf{f}}) - \mathcal{E}(h^*,\textbf{f}^*) \leq c_1 L\frac{G(\mathcal{H}(\textbf{X}))}{nT} \\ 
&\quad\quad + c_{2} B\frac{\sqrt{\lambda_{min}^{-1}}\sup_h \left\lVert h(\textbf{X}) \right\rVert}{n\sqrt{nT}} + \sqrt{\frac{8\ln{(4/\delta})}{nT}}
\end{aligned}
\label{eq:generalizatio error bound}
\end{equation}
\end{theorem}
\noindent where $c_1$, $c_2$ are universal constants, $G(\mathcal{H}(\textbf{X}))$ is the Gaussian average defined as $G(\mathcal{H}(\textbf{X})) = \mathbb{E}[\sup_{h\in \mathcal{H}} {\textstyle\sum}_{kti}\gamma_{kti}h(\mathbf{x}^{t}_{i})|\mathbf{x}^{t}_{i}]$, where $\{\gamma_{kti}\}$ is i.i.d standard normal variables. $L$ is the Laplacian matrix of graph with $T$ vertices and edge-weights $\{\omega_{ij}\}_{1\leq i < j \leq T}$, and $\lambda_{min}$ is its smallest non-zero eigenvalue. $B$ is any positive value that satisfies the condition ${\textstyle\sum}_{i,j=1}^T \omega_{ij}\cdot dist^{2}(\nabla_{A} f_i, \nabla_{A} f_j) \leq B^2$.

Some remarks over Theorem~\ref{thm:error bound}: \textbf{1).} The first term of the bound can be interpreted as the cost of estimating the shared representation learning function $h\in \mathcal{H}$. This term is typically of order $\frac{1}{n}$. The last term contains the confidence parameter. According to~\cite{maurer2016benefit} the constant $c_1$ and $c_2$ are pretty large, so the last term typically makes limited contribution in the bound. \textbf{2).} The second or middle term contains the cost of estimating task-specific predictors $f\in \mathcal{F}$, and this term is typically of order $\frac{1}{\sqrt{n}}$. Here the positive constant $B$ provides important insights into the relationship between our proposed regularizer and the error bound. \textbf{The smaller our regularization term becomes, the smaller values $B$ could take and in turn reduces the second term in the bound.} In general, our generalization error result bounds the gap between the test error of the model trained from finite samples and that trained from infinite data, namely the theoretically optimal model/solution. In other words, Theorem~\ref{thm:error bound} provides theoretical guarantee for our performance on actual test set.

\subsection{Relation with Other MTL Frameworks}
In this section, we mathematically elucidate the relation and difference between our proposed SRDML and other MTL methods, i.e., shallow MTL and deep MTL. Proof can be found in the appendix.

\noindent\textbf{Natural generalization of shallow MTL.} Following~\cite{zhang2021survey}, traditional multi-task learning methods (i.e., linear model based MTL) can be generally classified into two categories: \emph{multi-task feature learning} and \emph{multi-task relation learning}, with objective function $\min_{W,b,\Theta} L(W,b) + \lambda/2\cdot tr(W^{\intercal}\Theta^{-1}W)$ and $\min_{W,b,\Sigma} L(W,b)+\lambda/2 \cdot tr(W^{\intercal}\Sigma^{-1}W)$, where $\Theta$ and $\Sigma$ models the covariance between different features and tasks, respectively. For any regularization-based shallow MTL defined as above, it can be formulated as a \emph{special case} under the general framework of SRDML, with identity feature extraction function $h$, linear task-specific function $f$ and the corresponding regularizer on the input gradients.

\noindent\textbf{Relation with deep MTL.} Define two hyperparameters: \textbf{1).} The coefficient of regularizer in SRDML $\lambda$, and \textbf{2).} the number of layers $\ell$ before which the model is shared cross tasks. When $\lambda$ equals 0 and $\ell$ is greater than 1 and less than $L$ (total number of layers), SRDML degenerates to hard-parameter sharing. On the other hand, when $\ell$ equals to 1 and $\lambda$ is greater than 0, our SRDML is equivalent to soft-parameter sharing. Hence, both hard-parameter sharing and soft-parameter sharing can be formally formulated as special cases of our proposed SRDML framework.

\section{Experiments}
In this section, we validate SRDML on synthetic and real-world datasets against multiple methods, on various aspects including performance, sensitivity, qualitative analysis and ablation study. The experiments were performed on a 64-bit machine with 4-core Intel Xeon W-2123 @ 3.60GHz, 32GB memory and NVIDIA Quadro RTX 5000. \textbf{Code available at}~\url{https://github.com/BaiTheBest/SRDML}.

\begin{table}[t!]
\scriptsize
\centering
\caption{Attributes summary in CelebA and COCO.}
\vspace{-3mm}
\label{tab:attribute summary}
\begin{tabular}{c|ll||c|ll}
    \toprule
    T.id & CelebA &  COCO & T.id & CelebA &  COCO \\
    \midrule
    1 & ArchedEyebrows & person & 11 & PaleSkin & couch  \\
    2 & BagsUnderEyes & cat & 12 & Sideburns & bed \\
    3 & BlackHair & dog  & 13 & Smiling  & dining table \\
    4 & BrownHair & horse & 14 & WavyHair & laptop \\
    5 & Chubby & car & 15 & WearingLipstick  & tv  \\
    6 & DoubleChin & truck & 16 & Young & cell phone \\
    7 & Goatee & bus & 17 &   & bottle  \\
    8 & HeavyMakeup & motorcycle  & 18 & & cup  \\
    9 & MouthSlightlyOpen & bicycle & 19 &  & bowl  \\
    10 & Mustache & chair &   &   &   \\
    \bottomrule
\end{tabular}
\vspace{-2mm}
\end{table}

\subsection{Experimental Settings}

\begin{figure*}[t!]
\vspace{-0.1cm}
\centering
  \begin{subfigure}[b]{0.08\textwidth}
  %\hspace{0.4cm}
    \includegraphics[width=0.9\textwidth]{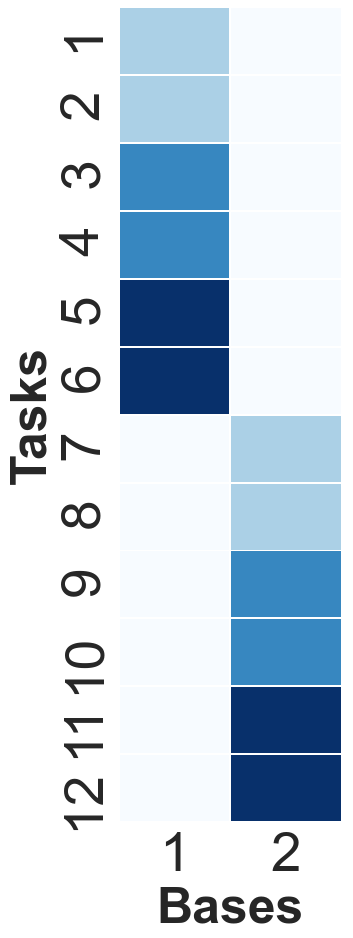}
    \caption{True W}
    \label{fig:ground truth w}
  \end{subfigure}
  %\hfill
  \begin{subfigure}[b]{0.28\textwidth}
  \hspace{0.3cm}
    \includegraphics[width=0.9\textwidth]{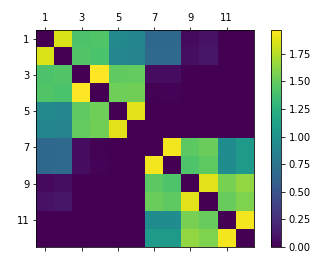}
    \caption{SRDML task relation}
    \label{fig:synthetic matrix}
  \end{subfigure}
  %\hfill
  \begin{subfigure}[b]{0.34\textwidth}
  %\hspace{-0.5cm}
    \includegraphics[width=\textwidth]{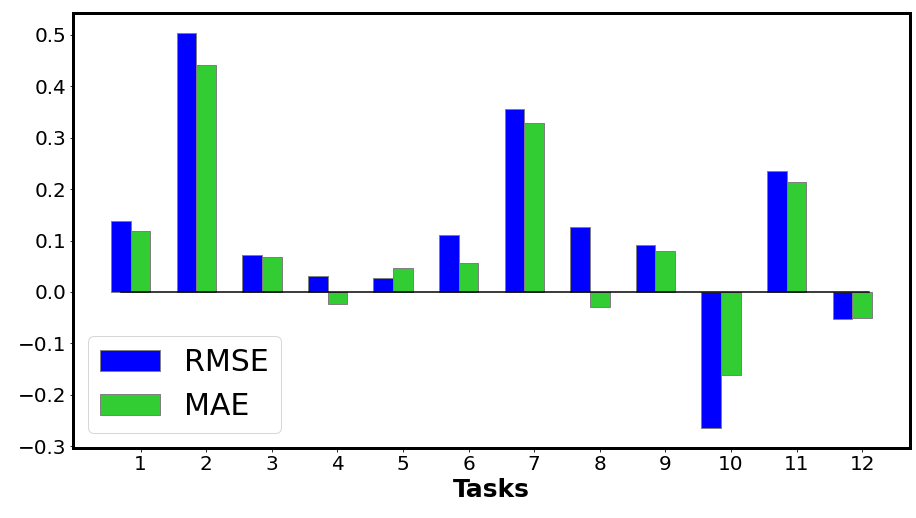}
    \caption{Per-task performance gain over STL.}
    \label{fig:synthetic mtl gain}
  \end{subfigure}
  \begin{subfigure}[b]{0.265\textwidth}
  %\hspace{-0.5cm}
    \includegraphics[width=\textwidth]{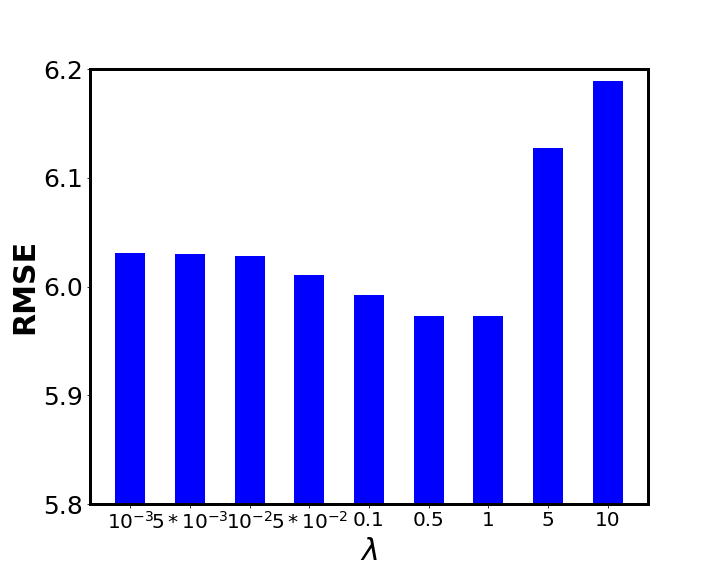}
    \caption{sensitivity analysis.}
    \label{fig:sensitivity}
  \end{subfigure}
  \caption{\textbf{Experimental results on synthetic dataset.} \textbf{(a)}: Ground-truth of each task's feature weight. \textbf{(b)}: Task relation learned by our proposed SRDML. Tasks from different bases show strong independency (as in dark purple), tasks from the same bases show clear similarities (as in light green), and each pair of twin tasks show very strong similarities (as in yellow). \textbf{(c)}: The performance improvement of SRDML over single task learning in RMSE (blue bar) and MAE (green bar). As shown, SRDML model generally outperforms STL on the synthetic dataset by a large margin. \textbf{(d)}: Sensitivity analysis on regularization coefficient.}
\vspace{-0.1cm}
\end{figure*}

\noindent\textbf{Controlled Synthetic Dataset.} We first check the validity of SRDML on a controlled regression synthetic dataset. We generate $T$ tasks ($T=12$) and for each task $i$ we generate $m$ samples ($m=100$). The input data $\textbf{X}_i\in \mathbb{R}^{m\times d}$ ($d=20$) for each task $i$ is generated from $\textbf{X}_i \sim \mathcal{N}(\eta_i,\textbf{I})$ with mean vector $\eta_i$ and identity covariance matrix $\textbf{I}$. Next, we generate feature weight $W$ by the following steps: \textbf{1)} Generate two base feature weights. As shown in Figure~\ref{fig:ground truth w}, the first base feature weight (on the LHS column) corresponds to $\textbf{w}_1=(\mathbf{1};\mathbf{0})^{\intercal}$ and the second base feature weight (on the RHS column) corresponds to $\textbf{w}_2=(\mathbf{0};\mathbf{1})^{\intercal}$, where $\mathbf{1}$ and $\mathbf{0}$ each denotes a 10-dimensional all-one and all-zero vector respectively. In this way, $\textbf{w}_1$ and $\textbf{w}_2$ can simulate two different regions in the input $X$ since the regions zeroed out by $\textbf{w}$ will not be helpful in corresponding tasks. \textbf{2)} Generate task specific feature weight. Based on $\textbf{w}_1$ and $\textbf{w}_2$, we further consider creating different levels of saliency by multiplying the base feature weights by some magnitude parameter. Here we select 3 different magnitude parameters to create different levels of saliency for each base feature weight, and for each level of saliency we create two tasks which are basically twin tasks. For example, in Figure~\ref{fig:ground truth w}, task 1 and task 2 are twin tasks which share the same level of saliency, and the lightest blue color means they are generated by the lowest magnitude parameter. We denote each generated task-specific feature weight as $w_i, \; i \in \{1,2,\cdots,T\}$. The aforementioned logistics are basically symmetric for $\textbf{w}_1$ and $\textbf{w}_2$. \textbf{3)} Add noise and create labels. We first inject some noise into each task's feature weight by randomly flipping the sign of the value in some positions of each $w_i$. The proportion of the flipped positions is controlled to guarantee the overall pattern can be well kept. Then, we generate the label for each task by $\textbf{Y}_i = \textbf{X}_i \cdot w_i + \mathbf{\epsilon}_i$, where $\mathbf{\epsilon}_i\sim \mathcal{N}(\mathbf{0},0.1\cdot\textbf{I})$ is random normal noise.

\noindent\textbf{Real-world Dataset.} We evaluate the proposed method on 3 real-world benchmarks with varying number of tasks and difficulty, including: multi-task version of CIFAR-10~\cite{krizhevsky2009learning} (\textbf{CIFAR-MTL}), a modified version of \textbf{CelebA}~\cite{liu2015faceattributes} and a modified version of \textbf{MS-COCO}~\cite{lin2014microsoft}. To follow our model's assumption, all tasks are image classification ones. For CIFAR-MTL, we follow existing work~\cite{rosenbaum2017routing} to create one task for each of the 10 classes in origianl CIFAR-10 dataset. There are 10 binary classification tasks with 2k training samples and 1k testing samples per task. CelebA has 200 thousand images of celebrity faces and each image is labeled with 40 facial attributes. We follow existing work~\cite{zhao2018modulation} to select 16 attributes more related to face appearance and ignore attributes around decoration such as eyeglasses and hat for our experiments. We randomly selected 30k training samples and include whole validation and test set. For MS-COCO we select 19 types of objects and remove those with too sparse labels. We include all images that contain at least two of the 19 types of objects and randomly split them into training and testing set by half. All results are reported on the test set. For hyperparameter tuning of our method, without further specification, we applied grid search on the range of $\{10^{-3}, 5*10^{-3},\cdots,0.5,1\}$ for the regularization coefficient.

\begin{table*}[t]
\centering
\small
%\vspace{-0.1cm}
\caption{Performance (\%) on real-world large-scale multi-task learning datasets. Our proposed SRMTL outperforms most comparison methods on all three datasets. Bold and underline score are for the best and second best methods, respectively.}
\vspace{-2mm}
    \begin{tabular}{l|cccc|cccc|cccc}
    \toprule
    \multirow{2}{*}{\textbf{Model}} & \multicolumn{4}{c}{\textbf{CIFAR-MTL}} & \multicolumn{4}{c}{\textbf{CelebA}} & \multicolumn{4}{c}{\textbf{COCO}}\\  \cline{2-13}
           & Accuracy   &   AUC   &  Precision     &  Recall  & Accuracy   &   AUC   &  Precision     &  Recall & Accuracy   &   AUC   &  Precision     &  Recall  \\ 
    \midrule 
     STL & 92.65 & 66.20 & 71.32 & 69.83 & 86.83 & 90.96 & 70.53 & 60.39 & 79.23  & 62.91 & 32.23 & 27.04 \\
     Hard-Share & 94.70 & 95.56 & 76.30 & 72.28 & 89.24 & 91.38 & 71.40 & 58.84 & 85.11 & 73.68 & \textbf{37.43} & 19.84 \\
     \hline
     Lasso & 91.48 & 86.64 & 68.90 & 24.74 & 76.55 & 66.69 & 37.38 & 36.62 & 78.36 & 64.40 & 28.53 & 28.61  \\
     L21 & 91.50 & 87.58 & 68.01 & 29.32 & 76.09 & 66.12 & 37.11 & 36.13 & 75.07 & 65.02 & 28.95 & 27.34  \\
     RMTL & 92.28 & 85.65 & 61.54 & 28.15 & 75.52 & 66.99 & 37.48 & 36.74 & 76.87 & 65.01 & 29.28 & 28.43   \\
     \hline
     MRN & 94.51 & 96.67 & 79.94 & \underline{76.95} & 89.35 & 91.54 & 71.51 & 64.64 & 85.13  & 75.88  & 32.73 & 25.89  \\
     MMoE & 93.53 & 93.17  & 73.42  & 69.32  &  77.57  &  67.84  &  68.79  &  58.92  &  81.20  &  62.37  &  33.08  & 26.14 \\
     PLE  & 94.01 & 93.32  &  75.26 & 70.15  &  83.21  &  69.32  &  70.03  &  59.72    &  82.53  &  63.42  &  35.27  & 27.53 \\
     MGDA-UB & 90.74 & 84.38 & 57.80 & 24.10 & 90.03 & 92.92 & 73.42 & 62.65 & 84.51 & 73.68 & \underline{36.17} & 16.08  \\
     PCGrad & 95.11 & \underline{96.69} & 79.03 & 74.82 & \underline{90.11} & 92.87 & 73.51 & 62.92 & 85.42 & 74.39 & 34.52 & 25.26   \\
     \hline
     SRDML & \underline{95.82} & 96.43 & \underline{81.22} & 75.93 & 90.15 & \underline{92.95} & \underline{73.87} & \underline{64.91} & \underline{85.68} & \underline{76.77} & 35.82 & \underline{28.73}  \\
     SRDML (w/. PCGrad) & \textbf{96.03} & \textbf{96.72} & \textbf{82.59} & \textbf{77.01} & \textbf{90.26} & \textbf{93.01} & \textbf{73.93} & \textbf{65.30} & \textbf{85.87} & \textbf{78.38} & 36.14 & \textbf{30.02}  \\
    \bottomrule
    \end{tabular}%
    \label{tab:results_clf}
    %\end{adjustwidth}
%\vspace{-0.1cm}
\end{table*}

\noindent\textbf{Comparison Methods} We compare SRDML with various existing methods, including two baselines, three shallow and five deep sate-of-the-art MTL methods: 

\begin{itemize}[leftmargin=*]
    \item \textbf{Practical Baselines:} \emph{1). Single Task Learning (STL)} is to train a separate predictor for each task independently. \emph{2) Hard Parameter Sharing (Hard-Share)} considers a shared representation learning backbone (e.g., convolutional layers in CNN) and task-specific prediction head.
    \item \textbf{Shallow MTL Methods:} \emph{1) Lasso} is an $\ell_1$-norm regularized method which introduce sparsity into the model to reduce model complexity and feature learning, and that the parameter controlling the sparsity is shared among all tasks. \emph{2) Joint Feature Learning (${L_{21}}$)}~\cite{evgeniou2007multi} assumes the tasks share a set of common features that represent the relatedness of multiple tasks. \emph{3) Robust Multi-task Learning (RMTL)}~\cite{chen2011integrating} method assumes that some tasks are more relevant than others. It assumes that the model $W$ can be decomposed into a low rank structure $L$ that captures task-relatedness and a group-sparse structure $S$ that detects outliers. 
    \item \textbf{Deep MTL Methods:} \emph{Multilinear Relationship Networks (MRNs)} places a tensor normal prior on task-specific layers of the deep multi-task learning model~\cite{long2017learning}. \emph{2) Multi-gate Mixture-of-Experts (MMoE)}~\cite{ma2018modeling} adapt the Mixture-ofExperts (MoE) structure to multi-task learning by sharing the expert submodels across all tasks, while also having a gating network. \emph{3) Progressive layered extraction (PLE)}~\cite{tang2020progressive} separates shared components and task-specific components explicitly and adopts a progressive routing mechanism to extract and separate deeper semantic knowledge gradually, improving efficiency of joint representation learning and information routing across tasks in a general setup. \emph{4) Multi-task Learning as Multi-Objective Optimization (MGDA-UB)}~\cite{sener2018multi} considers multi-task learning from optimization perspective by using Pareto optimality and Multiple Gradient Descent Algorithm. \emph{5) Gradient Surgery for Multi-task Learning (PCGrad)}~\cite{yu2020gradient} aims to solve the problem of gradient interference by gradient surgery, which is basically by gradient projection to make sure the gradients of different tasks have direction smaller than $90^{\circ}$. Since PCGrad targets gradient interference, it is only applied onto the shared layers of each model to avoid the contradiction of each task's gradients. Specifically, PCGrad is applied onto the shared convolutional layers.
\end{itemize}

\noindent\textbf{Implementation Details.} All shallow MTL methods are implemented according to standard package MALSAR~\cite{zhou2011malsar}. Deep MTL methods and our SRDML are built based on VGG-16~\cite{simonyan2014very}, which is a very popular architecture in computer vision. The convolutional layers are followed by one fully connected layer with 128 hidden neurons and one classification layer for our SRDML. Each model is trained by Adam~\cite{kingma2014adam}. For PCGrad, due to the fact that it is a gradient surgery method which is model-agnostic and can be applied onto any deep MTL method, we report its performance by combining it with the best baseline on each real-world dataset (i.e., Hard-Share on CIFAR-MTL, MGDA-UB on CelebA, MRN on COCO). In addition, we also consider applying PCGrad onto our own method SRDML, resulting in two versions of our method, namely SRDML and SRDML with PCGrad.

\subsection{Experimental Results}

\noindent\textbf{Effectiveness on controlled synthetic dataset.} The empirical results on the regression synthetic dataset demonstrate that our model can generally outperform single task learning and is capable to capture the ground-truth task relations. Quantitative evaluation in Figure~\ref{fig:synthetic mtl gain} shows that SRDML can outperform single task learning in general, which can be attributed to the effective knowledge sharing between task-specific layers. In addition, the task relationship pattern (i.e., $w_{ij}$ in Eq.~\ref{eq:objective funciton with normalization}) learned by SRDML as shown in Figure~\ref{fig:synthetic matrix} is accurate and reasonable, since tasks belong to different bases are well-separated and meanwhile each pair of twin tasks shows very strong correlation (corresponds to those yellow boxes). Within each base, different pairs of twin tasks also show relatively strong relationship due to the fact that they share the same base and only differ in the magnitude.

\noindent\textbf{Sensitivity analysis.} The sensitivity of hyperparameter $\lambda$ in SRDML on synthetic dataset is shown in Figure~\ref{fig:sensitivity}. As can be seen, the optimal value for $\lambda$ is around 0.5 meansured by RMSE. The general "U" shape is potentially reasonable because as $\lambda$ goes to infinity the regularization term would dominate the overall objective function while too small $\lambda$ will reduce the functionality of the regularizer and finally degenerate to single task learning.

% \begin{table}[t!]
% \small
% \centering
% %\vspace{-0.4cm}
% \begin{tabular}{lcccc}
% \hline
%  Model & Accuracy & AUC & Precision & Recall \\
% \hline
%  Lasso & 91.48 & 86.64 & 68.90 & 24.74 \\
%  L21 & 91.50 & 87.58 & 68.01 & 29.32 \\
%  RMTL & 92.28 & 85.65 & 61.54 & 28.15 \\
%  STL & 94.35 & 66.69 & 74,32 & 69.83 \\
%  TCGCN & 94.70 & 95.56 & 76.30 & 72.28 \\
%  MRN & 94.51 & 96.67 & 79.94 & \textbf{76.95} \\
%  MGDA-UB & 90.74 & 84.38 & 57.80 & 24.10 \\
%  PCGrad & 95.11 & \textbf{96.69} & 79.03 & 74.82 \\
%  \hline
%  SRDML (ours) & \textbf{95.82} & 96.43 & \textbf{81.22} & 75.93 \\
%  SRDML + PCGrad & \textbf{96.03} & \textbf{96.72} & \textbf{82.59} & \textbf{77.01} \\
% \hline
% \end{tabular}
% \caption{Performance on CIFAR-MTL.}
% \label{tab:cifar perforamnce}
% \end{table}

\begin{figure*}[t!]
\centering
  \begin{subfigure}[c]{0.485\textwidth}
    \includegraphics[width=\textwidth]{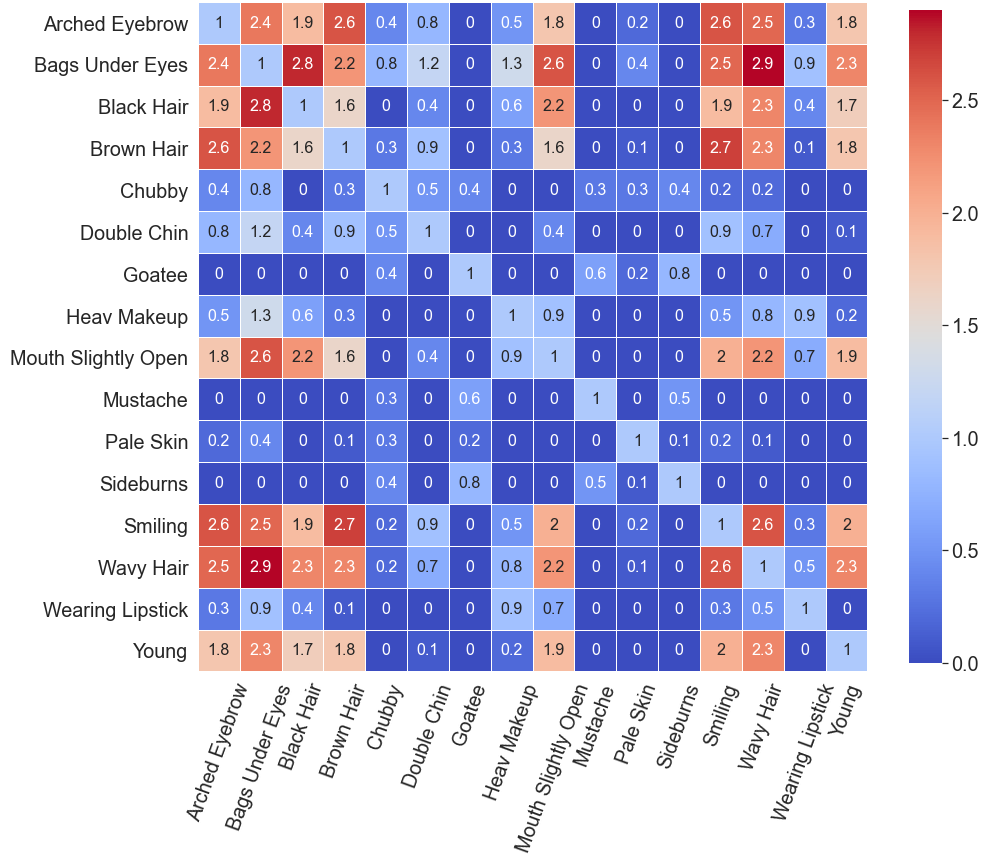}
    \caption{$\textbf{CelebA}$}
    \label{fig:f1}
  \end{subfigure}
  \hspace{0.3cm}
  %\hfill
  \begin{subfigure}[c]{0.48\textwidth}
    \includegraphics[width=\textwidth]{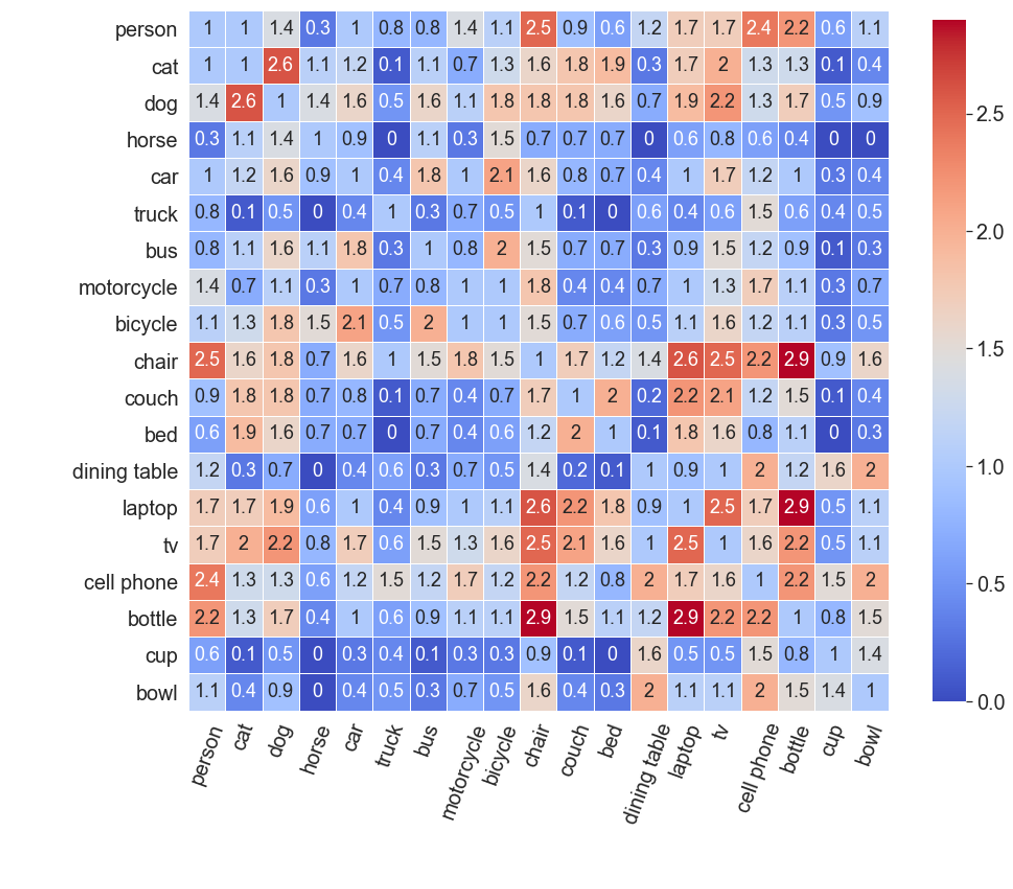}
    \caption{$\textbf{COCO}$}
    \label{fig:f2}
  \end{subfigure}
  \vspace{-0.2cm}
  \caption{Visualization of task relation learned by SRDML on real-world dataset. Zoom in for detail.}
   \vspace{-1mm}
  \label{fig:task relation on real dataset}
\end{figure*}

\noindent\textbf{Effectiveness on real-world datasets.} 
\begin{itemize}[leftmargin=*]
    \item \textbf{CIFAR-MTL:} Table~\ref{tab:results_clf} shows the performance results of our proposed SRDML and other baselines on CIFAR-MTL dataset. We can make the following observations from the results. \textbf{1).} Deep multi-task learning models generally outperform shallow ones by a great margin, which confirms the importance of learning the deep representation features as well as the shared policy of feature extraction part which allows knowledge transfer across tasks. \textbf{2).} Our proposed SRDML outperforms baselines in majority of metrics and achieves comparable performance in rest. \textbf{3).} Combining with PCGrad can further improve the performance of SRDML due to the mitigated negative transfer in the shared convolutional layers by gradient surgery of PCGrad.
    \item \textbf{CelebA:} In this case, we tackle a larger and more challenging benchmark, where we tailored the dataset to contain 16 binary classification tasks with each one corresponding to a certain human's facial feature. As shown in Table~\ref{tab:results_clf}, our model outperforms all comparison methods in majority of metrics, which is attributed to the potential fact that the salient regions in some tasks are close to those in the related tasks. For example, there are two tasks to classify whether a celebrity's beard is goatee or mustache, respectively. For both tasks the salient regions are highly overlapped around the mouth area (as can be seen in Section "Saliency map visualization" in appendix) so enforcing similar input gradients around the mouth area could improve the knowledge transfer and achieve better performance. 
    \item \textbf{COCO:} To evaluate our model under various settings, we consider COCO which contains different types of objects like human, animals, vehicles, furniture, etc, and each type object has varying rate of occurrence. In Table~\ref{tab:results_clf}, we report the task-average classification error with lower values indicating better performance. As shown in Table~\ref{tab:results_clf}, our proposed SRDML outperforms all the baselines by a great margin. This experiment also validates the effectiveness of our model when the number of tasks is relatively large and the image context is complicated. Moreover, MMoE and PLE perform generally not quite well probably due to the fact that these two approaches are designed for multi-task learning under recommender system scenario, which is not similar to that in multi-task image classification, e.g., the number of tasks in our case is much larger and hence more challenging.
\end{itemize}

\noindent\textbf{Qualitative analysis.} Here we demonstrate that SRDML can learn reasonable task relations on challenging real-world datasets by visualizing the task weight matrix (i.e., $w_{ij}$ in Eq.~\ref{eq:objective funciton with normalization}). As shown in Figure~\ref{fig:task relation on real dataset}, many highlighted task relations are intuitive. In CelebA, our proposed SRDML successfully learned the similarity of tasks sharing the same/similar regions around face, lie "Arched Eyebrow" and "Bags Under Eyes"; "Black Hair", "Brown Hair" and "Wavy Hair"; "Goatee", "Sideburns" and "mustache", etc. On the other hand, our model can also learn reasonable task similarities in COCO, including "cat" and "dog"; "car", "bus" and "bicycle"; "couch" and "bed", etc. We also conducted qualitative analysis experiment on the saliency map generated by our proposed SRDML on similar or related tasks. Please refer to the appendix for the detail.

\noindent\textbf{Adaptive regularizer on contradicting tasks.} In this section, we conducted another sensitivity analysis when all tasks compete (we generate such synthetic dataset by following a similar procedure introduced in Section 5.1), and the results in Table~\ref{tab:contro task} demonstrate the efficacy of our regularization term, which can adaptively decrease the task-similarity weight to zero and avoid competition.  

\begin{table}[t!]
\centering
%\vspace{-1mm}
\caption{Sensitivity analysis on regularizer coefficient when tasks are contradicting. Our regularizer coefficient can adaptively reduce to zero and avoid negative transfer.}
 \vspace{-2mm}
\begin{tabular}{ccccccc}
\toprule
 $\lambda$ & 1 & 0.1 & 0.01 & 0.001 & 0 \\
\hline
 RMSE.  & 2.726 & 1.550 & 1.405 & 1.393 & 1.392 \\
 MAE.   & 2.198 & 1.260 & 1.127 & 1.127 & 1.126 \\
\bottomrule
\end{tabular}
\vspace{-4mm}
\label{tab:contro task}
\end{table}

\noindent\textbf{Ablation study.} In this section, we present an ablation study on the task relation learning part in the regularizer. Specifically, we remove the $\{\omega_{ij}\}_{1\leq i < j \leq T}$ in Eq.~\ref{eq:objective funciton with regularizer} and the coefficient for each term in the regularizer is just the hyperparameter $\lambda$. We conducted experiments on all three real-world datasets to see the difference, and the results are shown in Table~\ref{tab:ablation study}.

\begin{table}[h!]
\small
\centering
\caption{Ablation study on adaptive regularizer (Accuracy)} 
 \vspace{-2mm}
\begin{tabular}{cccc}
\toprule
 & CIFAR-MTL & CelebA & MS-COCO \\
\hline
 SRDML. (\textbf{w/o} regularizer) & 94.92 & 89.74 & 85.18  \\
 SRDML. (\textbf{w/.} regularizer) & \textbf{95.82} & \textbf{90.15} & \textbf{85.68} \\
\bottomrule
\end{tabular}
\label{tab:ablation study}
\end{table}

\section{Conclusion}
Learning interpretable task relations is challenging in multi-task learning problem. In this paper, we proposed Saliency-regularized Deep Multi-task Learning (SRDML) framework, which regularizes the input gradient of different tasks by saliency and achieves good task relation interpretability. Instead of regularizing parameters like existing work, we directly regularize in functional space, which allows better expressiveness. Theoretical analyses show that our regularizer can help reduce the generalization error. Experiments on multiple synthetic and real-world datasets demonstrate the effect and efficiency of our methods in various metrics, compared with several comparison methods and baselines. The reasonableness of the task relation learned by SRDML is also validated on different challenging real-world datasets.

\section*{Acknowledgement}
    This work was supported by the National Science Foundation (NSF) Grant No. 1755850, No. 1841520, No. 2007716, No. 2007976, No. 1942594, No. 1907805, a Jeffress Memorial Trust Award, Amazon Research Award, NVIDIA GPU Grant, and Design Knowledge Company (subcontract number: 10827.002.120.04).

\bibliographystyle{ACM-Reference-Format}
\bibliography{srdml}

%%%%%%%%%%%%%%%%%%%%%%%%%%%%%%%%%%%%%%%%%%%%%%%%%%%%%%

\appendix
\newpage
%\onecolumn

\begin{center}
    {\huge \textbf{Appendix}} 
\end{center}

\noindent In this appendix, we describe detailed experimental setup, additional experimental results, and complete proofs. Our code is available at~\url{https://github.com/BaiTheBest/SRDML}. Please note that the code is subjected to reorganization to improve the readability.  

\section{Theoretical Proof}
In this section, we provide the formal proof for all the theories presented in Saliency-regularized Deep Multi-task Learning paper.

\subsection{Proof of Theomre 1}

\begin{proof}
Suppose $\mathcal{X}\subseteq \mathbb{R}^{K}$ is an open set and $f_1, f_2: \mathcal{X}\to\mathbb{R}$, where both functions are differentiable and equal to zero at the origin.

\noindent"$\Longrightarrow$": This direction is obvious, since two exactly the same functions will have the same gradient at any input point. 

\noindent"$\Longleftarrow$": Given $\nabla f_1 (\mathbf{x}) = \nabla f_2 (\mathbf{x})$, we know that
\begin{equation}
    \partial f_1/\partial x_k = \partial f_2/\partial x_k, \; k = 1,2,\cdots,K, \;\forall\; \mathbf{x}\in \mathcal{X}
\end{equation}
For arbitrary $k$, by $\partial f_1/\partial x_k = \partial f_2/\partial x_k$, we know that
\begin{equation}
    \exists\;c_k (x_1,\cdots,x_{k-1},x_{k+1},\cdots,x_K),\; s.t.,\; f_1 = f_2 + c_k
\end{equation}
Meanwhile, notice $\forall\; l\neq k,\quad \partial c_k / \partial x_l = 0$ (otherwise, contradiction!) Hence, $d c_k = 0$ and we know $c_k$ is a constant. Also, the value of $c_k$ does not depend on $k$ since for all $k,l$, we have $f_1 - f_2 = c_k = c_l$, thus there exists a constant $c$ such that $f_1 = f_2 + c$. Finally, by the boundary condition that $f_1 (\mathbf{0}) = f_2 (\mathbf{0}) = 0$, we know that $c = 0$, i.e., $f_1 = f_2$, which finishes the proof.
\end{proof}

\subsection{Proof of Theorem 2}

In this section, we provide the proof of our model's generalization error bound. First, we introduce some definitions and lemmas which will be continuously used, and at the end of this section we present the proof for Theorem~\ref{thm:error bound}.

In general, we will use $\gamma$ to denote a generic vector of i.i.d standard normal variables, whose dimension will be clear in context. In addition, without further specification, we will use $K$, $T$, $n$ to denote the (flattened) dimension of the output space from the feature extraction function $h$, number of tasks, and number of training samples, respectively. We denote the representation class for task-specific function $f$ and representation extraction function $h$ as $\mathcal{F}$ and $\mathcal{H}$, respectively. Two hypothesis classes here can be very general, and the only assumption here is that $\forall f\in \mathcal{F}$, $f$ has Lipschitz constant at most L, for any positive L. 

\begin{definition}
Given a set $V \subseteq \mathbb{R}^{n}$, define the Gaussian average of $V$ as
\begin{equation}
    G(V) \coloneqq \mathbb{E}\sup_{v\in V} \langle \gamma, v\rangle = \mathbb{E}\sup_{v\in V}\displaystyle\sum_{i=1}^{n} \gamma_i v_i
\label{eq:gaussian average}
\end{equation}
\end{definition}

As mentioned in section 3.1 in main paper, we denote the feature representation learning part as function $h\in \mathcal{H}$. As we will see later, the complexity of representation class $\mathcal{H}$ is important in our proof for the error bound, so we define a measure of its complexity by Gaussian average.
\begin{definition}
Given observed input data $\mathbf{X}\in \mathcal{X}^{Tn}$, define a random set $\mathcal{H}(\mathbf{X})\subseteq\mathbb{R}^{KTn}$ by
\begin{equation}
    \mathcal{H}(\mathbf{X}) \coloneqq \left\{(h_{k}(\mathbf{x}_{i}^{t})): h\in \mathcal{H} \right\}.
\end{equation}
The Gaussian average over $\mathcal{H}(\mathbf{X})$ can be defined accordingly as
\begin{equation}
    G(\mathcal{H}(\mathbf{X})) = \mathbb{E}[\sup_{h\in\mathcal{H}}\displaystyle\sum_{kti}^{K,T,n}\gamma_{kti}h_{k}(\mathbf{x}_{ti})|\mathbf{x}_{ti}]
\end{equation}
\end{definition}

The following lemmas are useful in our proof later, and we introduce them here in advance.

\begin{lemma}
\label{lem:trace}
$\forall\; A, C\in\mathbb{R}^{m\times n}$ and $B\in\mathbb{R}^{m\times m}$, 
\begin{equation}
    tr(A^{\intercal}BC) = \displaystyle\sum_{i,j}^{m} B_{ij}\displaystyle\sum_{k=1}^{n} A_{ik}C_{jk}.
\end{equation}
\end{lemma}
\vspace{-0.4cm}

\begin{lemma}
\label{lem:differntial and gradient}
Suppose $\mathcal{X}\subseteq \mathbb{R}^{K}$ is an open set, and two differentiable functions $f_1, f_2\; : \mathcal{X}\rightarrow\mathbb{R}$. $\forall x\in \mathcal{X}$, if
\begin{equation}
    \exists\; B > 0, \;\; s.t\; \left\lVert \nabla f_1 (x) - \nabla f_2 (x) \right\rVert < B
\end{equation}
then
\begin{equation}
    \displaystyle\lim_{\Delta x\to 0} \left\vert \frac{f_{1}(x+\Delta x)-f_{1}(x)}{\left\lVert \Delta x \right\rVert} - \frac{f_{2}(x+\Delta x)-f_{2}(x)}{\left\lVert \Delta x \right\rVert} \right\vert < B.
\end{equation}
\end{lemma}

Given everything above, we can prove our Theorem~\ref{thm:error bound}. However, the formal proof of Theorem~\ref{thm:error bound} is quite complicated and due to the limited space of appendix here, we decide to put the formal proof for our Theorem~\ref{thm:error bound} into the \textbf{link}:~\url{https://drive.google.com/file/d/1Mtbf5zpftIP9F31V5vgsXUU07tYjf0Ad/view?usp=sharing}. Please refer to the link for our formal proof.

\subsection{Proof of Section 4.2}

\textbf{Natural generalization of shallow MTL}

\begin{proof}
Basically, when the feature extraction function $h$ is identity function and each task-specific function $f_t$, $t=1,2,\cdots,T$ are linear functions, we know for any input $x\in\mathcal{X}$,
\begin{equation}
    h(x) = x, \quad \nabla f_t (x) = w_t,\; \forall t
\end{equation}
where $w_t$ is the model parameter of linear model $f_t$. Hence, denote $W=[w_1;w_2;\cdots;w_T]$ and take the $dist()$ function in Eq.~\ref{eq:objective funciton with regularizer} to be inner product, by Lemma~\ref{lem:trace} we have
\begin{equation}
    \begin{split}
        \sum\nolimits_{i,j} \omega_{ij} \cdot dist(\nabla f_i (x), \nabla f_j (x)) &= \sum\nolimits_{i,j} \omega_{ij} \cdot \langle w_i, w_j \rangle \\ &= tr(W^{\intercal}\Omega W)
    \end{split}
\end{equation}
where $\Omega = (\omega_{ij})$. Let $\Omega$ to be either $\Theta^{-1}$ or $\Sigma^{-1}$ as in section 4.2 can finish the proof.

\end{proof}

\noindent\textbf{Relation with deep MTL}

\begin{proof}
First, we define two hyperparameters:
\begin{itemize}
    \item $\lambda$: The coefficient of our regularizer in SRDML
    \item $\ell$: The index of the layer before which the model is shared cross different tasks.
\end{itemize}

\textbf{Case 1.} If $\lambda=0$ and $1 < \ell < L$, where $L$ (please differentiate this $L$ with that for Lipschitz constant) denotes the total number of layers, our SRDML has no regularization and is simply equivalent to hard-parameter sharing.

\textbf{Case 2.} If $\lambda > 0$ and $\ell = 1$, each layer in our SRDML is separate for different tasks and the regularization is posed on all the layers, which is equivalent to soft-parameter sharing.

\end{proof}

\section{Additional Details on Synthetic Dataset Generation}

\paragraph{“What are base feature weights?”} Since we want to generate tasks with different level of similarity in our synthetic dataset, we achieved it by controlling the similarity in the feature weight (i.e., w) of different tasks. The base feature weights $w_1$ and $w_2$ are basically two vectors (with length equal to number of features) for generating the feature weight vectors for all the tasks. We call them “base” feature weight because they serve as the base vector or unit vector for generating all the tasks’ feature weights. In addition, $w_1$ and $w_2$ are orthogonal and each has length 1 in any dimension.

\paragraph{“How are base feature weights used?”} The base feature weights are used to generate each task’s feature weight in the following steps: 1) We choose which base the current task belongs to. In our setting, we chose the first half of tasks to belong to the first base (i.e., $w_1$) and the second half of tasks to belong to the second base (i.e., $w_2$). Since two bases are orthogonal, they can actually simulate two non-overlapping regions in pictures which means tasks from different bases should not be similar while those from the same base should be similar since, they focus on the same region. 2) Within each base, we multiply the base vector (i.e., $w_1$ and $w_2$) by some positive integers to generate the actual feature weight for the tasks. For example, we multiply $w_1$ by integer 1, 2 and 3 to generate the feature weight vectors for the first half of tasks.

\section{Normalization on Input Gradient}

We add an experiment on our method with normalizing the input gradients and compare its results with our original method (i.e., without normalization) on ALL 3 real-world dataset we used in our original paper. As shown in the Table~\ref{tab:effect of normalization}, adding normalization did not obviously change the performance in task-average classification error. The task-average classification error decreased by $<$ 0.2\% on CIFAR-MTL and increased by around 0.1\% on CelebA and COCO. One explanation is, for similar tasks like “Black hair” and “Brown hair” in CelebA, we empirically observed that the magnitude for the gradients was close to each other, which might limit the point in applying gradient normalization in such case.

\begin{table}[h!]
\centering
\small
\caption{Normalization of input gradient} 
\begin{tabular}{lccc}
\toprule
 & CIFAR-MTL & CelebA & MS-COCO \\
\hline
 SRDML \textbf{w/o} normalization & 4.18 & 9.91 & 14.32  \\
 SRDML \textbf{w/} normalization & 4.02 & 10.03 & 14.41 \\
\bottomrule
\end{tabular}
\label{tab:effect of normalization}
\end{table}

\section{Additional Qualitative Analysis}

\begin{figure}[h!]
\centering
  \begin{subfigure}[b]{0.15\textwidth}
  %\hspace{-0.2cm}
    \includegraphics[width=0.9\textwidth]{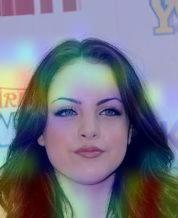}
    \caption{black hair}
    \label{fig:black hair}
  \end{subfigure}
  %\hfill
  \begin{subfigure}[b]{0.15\textwidth}
  %\hspace{0.3cm}
    \includegraphics[width=0.9\textwidth]{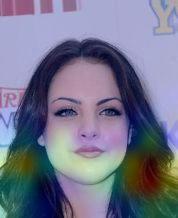}
    \caption{brown hair}
    \label{fig:brown hair}
  \end{subfigure}
  \vspace{-3mm}
  \caption{Saliency map generated by SRDML for hair tasks.}
  \label{fig:hair}
\end{figure}
%\vspace{-2mm}

\begin{figure}[t!]
 \vspace{-2mm}
\centering
  \begin{subfigure}[b]{0.12\textwidth}
  %\hspace{-0.2cm}
    \includegraphics[width=0.9\textwidth]{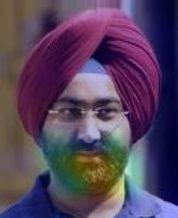}
    \caption{goatee}
    \label{fig:goatee}
  \end{subfigure}
  %\hfill
  \begin{subfigure}[b]{0.12\textwidth}
  %\hspace{0.3cm}
    \includegraphics[width=0.9\textwidth]{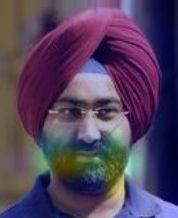}
    \caption{mustache}
    \label{fig:mustache}
  \end{subfigure}
  \begin{subfigure}[b]{0.12\textwidth}
  %\hspace{0.3cm}
    \includegraphics[width=0.9\textwidth]{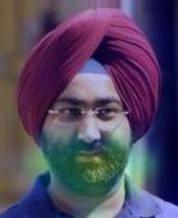}
    \caption{no beard}
    \label{fig:no beard}
  \end{subfigure}
   \vspace{-2mm}
  \caption{Saliency map generated by SRDML for beard  tasks.}
  \label{fig:beard}
  \vspace{-4mm}
\end{figure}

We also conduct a set of qualitative analysis experiment on the saliency map generated by our proposed SRDML on similar or related tasks. As can be seen in Figure~\ref{fig:hair} and Figure~\ref{fig:beard}, our proposed SRDML can generate saliency map focusing on similar regions for related tasks. For example, the saliency map generated for "Black hair" and "Brown hair" both generally overlap around the hair region of the woman, and the saliency map generated for three types of beard all overlap around the mouth and beard region of the man. Notice that the quality of saliency itself is not the main focus of this paper, but instead we are more interested in the task relation induced by the saliency map similarity (i.e., saliency across tasks).

\section{Additional Remarks on Theorem~\ref{thm:error bound}}

In this section, we provide more remarks on our main theorem, namely Theorem~\ref{thm:error bound}, for better understanding.

\noindent\textbf{Remark 1.} The equation above bounds the gap between the test error of the model trained from finite samples and that trained from infinite data, namely the theoretically optimal model/solution. In other words, Theorem 2 provides theoretical guarantee for our performance on actual test error.

\noindent\textbf{Remark 2.} In Eq.~\ref{eq:objective funciton with constraint with relaxation}-~\ref{eq:objective funciton 3}, we assume all tasks share the same set of X which is a very common case in Multi-task Learning on image dataset. Theorem 2 does not need different tasks to have different X(t), since $\mu_1 = \mu_2 = \cdots = \mu_T = \mu$ is a special case of the version in Theorem~\ref{thm:error bound}. Our current assumption is actually a more general one and can handle the case in Eq.~\ref{eq:objective funciton with constraint with relaxation}-~\ref{eq:objective funciton 3}.

\end{document}